\documentclass[11pt]{article}

\usepackage[margin=1.15in]{geometry}
\usepackage{amsmath,amssymb,amsthm}
\usepackage{mathtools}
\usepackage{graphicx}
\usepackage{booktabs}
\usepackage{microtype}
\usepackage{xcolor}
\usepackage[colorlinks=true,linkcolor=blue!55!black,citecolor=blue!55!black,urlcolor=blue!55!black]{hyperref}

\theoremstyle{plain}
\newtheorem{theorem}{Theorem}
\newtheorem{lemma}[theorem]{Lemma}
\newtheorem{proposition}[theorem]{Proposition}
\newtheorem{corollary}[theorem]{Corollary}
\theoremstyle{definition}
\newtheorem{definition}[theorem]{Definition}
\newtheorem{problem}[theorem]{Problem}
\newtheorem{remark}[theorem]{Remark}
\newtheorem{openproblem}[theorem]{Open Problem}

\DeclareMathOperator{\VC}{VCdim}
\DeclareMathOperator{\Ldim}{Ldim}
\DeclareMathOperator{\isdim}{isdim}
\DeclareMathOperator{\Err}{Err}
\DeclareMathOperator{\ErrReg}{ErrReg}
\DeclareMathOperator{\supp}{supp}
\newcommand{\Htwo}{H_2}
\newcommand{\X}{\mathcal{X}}
\newcommand{\Hy}{\mathcal{H}}
\newcommand{\D}{\mathcal{D}}
\newcommand{\hstar}{h^{\star}}
\newcommand{\Hle}[1]{\Hy_{\leq #1}}
\newcommand{\Hprime}{\Hy_{\mathrm{prime}}}
\newcommand{\R}{\mathbb{R}}
\newcommand{\N}{\mathbb{N}}
\newcommand{\E}{\mathbb{E}}
\newcommand{\pr}{\mathbb{P}}

\title{\bf When Does More Correct Data Hurt?\\ Insertion-Stability and the Limits of Dimension-Based
Theory}
\author{Joseph Sankoorikal Johny\\ \normalsize Independent Researcher, Seattle, USA}
\date{\today}

\begin{document}
\maketitle

\begin{abstract}
Adding data known to be correct ought to be safe. Not always. Larsen, Pabbaraju and Shetty model the
failure with a \emph{monotone adversary}, which reads an i.i.d.\ training sample and may append as
many further examples as it likes, provided the target hypothesis labels them all. Mehrotra has since
settled the cost, showing that for classes of VC dimension $d \geq 2$ no learner can guarantee
expected error better than $\Theta\big((d/n)\log(en/d)\big)$, a logarithmic factor above the clean
PAC rate.

Because that rate is a worst case over all classes, it says nothing about which classes actually
suffer the penalty, and the answer turns on the learner. We call a learner \emph{insertion-stable} if
feeding it more correctly labeled examples can only shrink the region where it errs. Such learners
are immune to the adversary, since on any given sample the risk after insertions never exceeds the
risk on the clean part alone, however much is added and however cleverly it is chosen.
High-probability guarantees carry over unchanged, and because Closure is insertion-stable every
intersection-closed class keeps its clean rate of $\E[\Err] \leq (21d+34)/n$.

Immunity is not something the classical dimensions can predict. Two classes can agree on
$\VC = \Ldim = 2$ and still split, one at $\Theta(1/n)$ and the other at $\Theta(\log(en)/n)$, while
intervals have unbounded Littlestone dimension and are immune anyway. On Mehrotra's hard class we
prove more than the failure of a single algorithm, showing that no monotone permutation-invariant
compression scheme of any finite size attains the clean rate.

The question is therefore not whether a class is hard, nor whether a learner is good, but whether
the two suit each other. Given an insertion-stable learner that is optimal on clean data, correct
additions are free, and without one the cost belongs to the class, so changing the learner will not
avoid it.
\end{abstract}

\section{Introduction}

Suppose you augment a training set: for every image you already hold, you add a rotated copy carrying
the same label. Or you notice a rare class is thin, so you duplicate its examples. Nothing you added
is wrong, and you chose what to add only after looking at what you had. Practitioners perform such
passes routinely and expect at worst no harm from them, since the standard generalization bounds
require independence or exchangeability and augmentation appears to threaten neither the labels nor
the sample's coverage. It does, however, destroy exactly that symmetry, and the guarantees rest on
it.

The model of Larsen, Pabbaraju, and Shetty~\cite{LPS26} isolates this effect from every other way a
dataset can go wrong. Their adversary reads the clean sample, then appends whatever examples it
pleases under a single restriction: each appended point must receive the label the target hypothesis
gives it. Nothing is mislabeled, so the added constraints are all valid, and an oracle identifying
the original points would let the learner ignore the rest at no cost. What the additions destroy is
independence: they are chosen with knowledge of the clean draw, and the learner receives one shuffled
pool in which the two groups are indistinguishable, so it can neither isolate the clean subsample nor
treat what it holds as exchangeable.

Mehrotra~\cite{Meh26} determined the optimal rate in this model. Over classes of VC dimension
$d \geq 2$ it is $\Theta\big((\min\{d,n\}/n)\log(en/\min\{d,n\})\big)$, exceeding the clean rate
$\Theta(d/n)$ by a logarithmic factor, against all randomized possibly improper learners; the same
rates hold with Littlestone dimension in place of $d$, so finite mistake bounds do not help either.
Counterintuitively, correctly labeled insertions raise the statistical difficulty of the problem.

\paragraph{This paper.} Mehrotra's bound is a supremum over classes, witnessed by a specific
construction, and his first stated open problem asks:
\begin{quote}
\emph{``Which structural properties of a fixed binary class determine its rate under monotone
insertions?''}
\end{quote}
Our answer is two-sided, and the two sides are what the worst-case rate obscures. What decides the
outcome is neither the class's size nor its classical dimensions, but the \emph{geometry of the error
region of the learner one runs on it}. Where a learner exists whose error region can only shrink as
correct examples arrive, the insertions cost nothing at all: not merely a survivable amount, but
nothing, at any budget. Where no such learner exists, the penalty is a property of the class, and
Mehrotra's lower bound says every learner pays it. The phenomenon therefore attaches to the
pairing of learner and class.

\paragraph{Contributions.}
\begin{enumerate}
\item \textbf{The insertion-stability lemma} (Section~\ref{sec:c1}). If a learner's error region
only shrinks under correctly labeled insertions, its adversarial risk is dominated \emph{almost
surely} by its clean risk. Adaptivity, arbitrarily many insertions, repeated and out-of-support
points, randomized adversaries, and high-probability guarantees are all free consequences. The proof
is a reduction; the i.i.d.\ assumption is quarantined inside a citation about the clean sample and
never applied to the corrupted one.

\item \textbf{Closure is insertion-stable} (Section~\ref{sec:c2}), so every intersection-closed
class retains the clean rate under arbitrarily many adaptive monotone insertions, with explicit
constants $(21d+34)/n$ in expectation. We also observe that Mehrotra's own bounded-budget converse
pads with \emph{negative} repeated points, and is therefore not merely survivable but
\emph{invisible} to Closure.

\item \textbf{Dimension independence} (Section~\ref{sec:c3}). $\Hle{2}$ and $\Hprime$ have identical
VC and Littlestone dimension and rates differing by $\Theta(\log n)$; intervals have unbounded
Littlestone dimension and retain the clean rate. No function of $\VC$ and $\Ldim$ determines the
monotone-insertion rate.

\item \textbf{The negative half, sharpened to a family of algorithms} (Section~\ref{sec:c3}). For
$\Hprime$ we show that \emph{no} consistent monotone permutation-invariant compression scheme, of
any finite size, attains the clean rate. This rules out a whole class of would-be remedies rather
than one algorithm, and is our main technical result.
\end{enumerate}

\paragraph{What we do not claim.} We give a sufficient condition, not a characterization. A converse, in which structure that is \emph{not} insertion-stable forces the logarithmic penalty,
remains open,
and Section~\ref{sec:open} says precisely what is missing. The candidate measure of
Section~\ref{sec:c4} is a bookkeeping device, not a theory: it repackages the hypotheses of a known
theorem, and we are explicit that it cannot currently be evaluated on a class without already
knowing that class's rate. Our machinery is borrowed from Hanneke~\cite{Han16b}; the contribution is
the reduction that makes it apply here, plus the negative result above. All combinatorial claims
about specific classes were additionally verified by exhaustive computation.

\section{Model and preliminaries}\label{sec:model}

Notation and the model statement follow~\cite{Meh26}. For a distribution $\D$ on $\X$ and
$g,h : \X \to \{0,1\}$ write $\Err_\D(g,h) = \pr_{X \sim \D}\{g(X) \neq h(X)\}$. Multiset union is
written $\uplus$, and every duplicate of a point counts as its own occurrence throughout.
Logarithms are base $e$.

\begin{problem}[Learning with a monotone adversary~\cite{LPS26,Meh26}]\label{prob:model}
Fix a nonempty class $\Hy \subseteq \{0,1\}^\X$, integers $n \geq 1$, $m \geq 0$, a target
$\hstar \in \Hy$, and a distribution $\D$ on $\X$. A deterministic monotone adversary with exact
budget $m$ is a map $A : \X^n \to \X^m$. A dataset is generated as follows.
\begin{enumerate}
\item Nature draws an ordered clean sample $X = (X_1,\dots,X_n) \sim \D^n$.
\item The adversary sees $X$ and returns $A(X) = (\widetilde X_1, \dots, \widetilde X_m)$. It may
repeat points and choose points outside $\supp(\D)$, but must label every appended point by
$\hstar$.
\item Nature draws an independent uniform permutation $\Pi$ of the $n+m$ occurrences and gives the
learner
$T_{\hstar,A}(X,\Pi) = \Pi\big( (X_i, \hstar(X_i))_{i=1}^n \uplus (\widetilde X_j, \hstar(\widetilde
X_j))_{j=1}^m \big)$.
\end{enumerate}
The learner knows $\Hy$, $n$, and $m$, but not which examples came from $\D$. It is scored by
$\Err_\D(\hat h, \hstar)$ on a fresh independent point from $\D$.
\end{problem}

Four features of the model are load-bearing below, and all are explicit in~\cite{Meh26}: every label
is correct; the adversary may repeat points and use points outside $\supp(\D)$; the learner sees a
uniform shuffle, so it is a function of the \emph{multiset} rather than the sequence; and the
adversary is \emph{adaptive}, seeing all of $X$ before choosing insertions. The last is the entire
source of difficulty; in the oblivious case the clean rate $\Theta(d/n)$ is already
attainable~\cite{LPS26}. Adversary randomization does not change the picture: as
Mehrotra notes, upper bounds hold after conditioning on the adversary's randomness.

A \emph{labeled multiset} is a finite multiset of pairs $(x,y) \in \X \times \{0,1\}$; it is
\emph{correctly labeled by} $\hstar$ if $y = \hstar(x)$ for every occurrence. We write $S \subseteq
T$ for multiset containment, so $T$ may contain additional copies of points already in $S$. The
\emph{version space} of $\Hy$ on a labeled multiset $L$ is $V_\Hy(L) = \{h \in \Hy : h(x) = y \text{
for all } (x,y) \in L\}$.

Recall that $\Hy$ is \emph{intersection-closed} if $\{\{x : h(x)=1\} : h \in \Hy\}$ is closed under
pairwise intersection of its members; conjunctions on $\{0,1\}^p$, axis-aligned rectangles on
$\R^p$, and $\Hle{d} = \{h : |\{x : h(x)=1\}| \leq d\}$ are all intersection-closed~\cite{Han16b}.
(We use the standard convention of closure under pairwise intersection of members; some authors
additionally require $\X$ itself to belong to the family, which $\Hle{d}$ would fail for
$|\X| > d$.) For such a class, the \emph{Closure} algorithm on a labeled multiset $L$ with
$V_\Hy(L) \neq \emptyset$ returns the classifier $\hat h_L$ with
\begin{equation}\label{eq:closure}
\{x : \hat h_L(x) = 1\} \;=\; \textstyle\bigcap_{h \in V_\Hy(L)} \{x : h(x) = 1\} .
\end{equation}

\section{Insertion-stability}\label{sec:c1}

For a learner $A$ and a labeled multiset $T$ correctly labeled by $\hstar$, define the
\emph{error region} $\ErrReg_{\hstar}(A(T)) = \{x \in \X : A(T)(x) \neq \hstar(x)\}$.

\begin{definition}[Insertion-stable]\label{def:is}
A learner $A$ is \emph{insertion-stable} if it is
\begin{enumerate}
\item[(i)] \emph{permutation-invariant}: $A(T)$ depends only on the labeled multiset $T$; and
\item[(ii)] \emph{superset-monotone}: for every $\hstar \in \Hy$ and all finite labeled multisets
$S \subseteq T$ both correctly labeled by $\hstar$,
\[
\ErrReg_{\hstar}(A(T)) \;\subseteq\; \ErrReg_{\hstar}(A(S)) .
\]
\end{enumerate}
For randomized $A$ we require (ii) for each realization of coins drawn independently of the sample.
\end{definition}

Three details matter and are deliberately permitted. Containment is \emph{multiset} containment, so
$T$ may repeat points of $S$, which is necessary because Mehrotra's constructions pad with many copies of
a point. No constraint ties $T$'s points to $\supp(\D)$; indeed $\D$ does not appear in
Definition~\ref{def:is} at all, which is what makes the reduction below work for all $\D$ at once.
And permutation-invariance is folded into the definition rather than carried alongside, because the
model's shuffle supplies it.

\begin{theorem}[Insertion-stability lemma]\label{thm:c1}
Let $A$ be insertion-stable. Then for every $\D$, every $\hstar \in \Hy$, every finite budget
$m \in \N$, and every possibly randomized adaptive monotone adversary, on a single
probability space carrying the clean draw, the adversary's coins, the shuffle, and the learner's
coins,
\[
\Err_\D\big(A(T),\hstar\big) \;\leq\; \Err_\D\big(A(S),\hstar\big) \qquad \text{almost surely,}
\]
where $S$ is the clean sample labeled by $\hstar$ and $T$ is the learner's input.
\end{theorem}

\begin{proof}
Condition on a realization of the clean sample $X$, the adversary's coins, the shuffle $\Pi$, and
the learner's coins.

\emph{Step 1: $T$ is a correctly labeled superset of $S$.} By Problem~\ref{prob:model}, $T$ consists
of the $n$ clean occurrences together with the $m$ inserted ones, permuted. Every inserted
occurrence is labeled by $\hstar$, as is every clean occurrence, so $T$ is correctly labeled by
$\hstar$; discarding the inserted occurrences exhibits $S \subseteq T$ as multisets.

\emph{Step 2: the error region shrinks.} By (i), $A(T)$ does not depend on $\Pi$ and is a function
of the labeled multiset alone. Applying (ii) to $S \subseteq T$ gives
$\ErrReg_{\hstar}(A(T)) \subseteq \ErrReg_{\hstar}(A(S))$.

\emph{Step 3: monotonicity of the measure.} Since $\D$ is a measure, and since
$\Err_\D(A(L),\hstar)$ equals $\D\big(\ErrReg_{\hstar}(A(L))\big)$ for any $L$, Step 2 yields
\[
\Err_\D(A(T),\hstar) \;\le\; \Err_\D(A(S),\hstar)
\]
for this realization.

Since the realization was arbitrary, the inequality holds almost surely.
\end{proof}

\begin{remark}[Unbounded budgets]\label{rem:sup}
Problem~\ref{prob:model} fixes a finite $m$, since a uniform permutation of $n+m$ occurrences is
undefined otherwise; Definition~\ref{def:is}(ii) likewise quantifies over finite multisets. We
therefore define the \emph{unbounded-budget risk} of a learner as
\[
R^{\infty}_\D(A,\hstar) \;:=\; \sup_{m \in \N}\; \E\big[\Err_\D(A(T_m),\hstar)\big],
\]
the supremum of its finite-budget risks. Theorem~\ref{thm:c1} bounds every term of that supremum by
the same quantity $\E[\Err_\D(A(S),\hstar)]$, which does not involve $m$ at all; the bound therefore
passes to the supremum unchanged. What superset-monotonicity buys is this uniformity in $m$, not a
literal infinite sample, and it is what we mean below by ``arbitrarily many insertions''.
The distinction matters because a budget-dependent bound would be useless here: Mehrotra's
lower-bound construction uses a budget growing rapidly with $n$, against which
his own Remark 3.2 bound $O(d(m{+}1)/n)$ is vacuous.
\end{remark}

\begin{corollary}[Transfer]\label{cor:transfer}
Let $A$ be insertion-stable. Because Theorem~\ref{thm:c1} is an almost-sure domination under an
explicit coupling, since the two learners are run on the \emph{same} clean sample, every monotone
functional of the risk transfers from the clean setting to the adversarial one. In particular, for
every finite budget $m$ and every randomized adaptive monotone adversary, hence also for
$R^{\infty}_\D$ of Remark~\ref{rem:sup}:
if $\E[\Err_\D(A(S),\hstar)] \leq B(n)$ then $\E[\Err_\D(A(T),\hstar)] \leq B(n)$; and if
$\Err_\D(A(S),\hstar) \leq B(n,\delta)$ with probability at least $1-\delta$, then the same bound
holds with probability at least $1-\delta$ against the adversary. All moments and tail probabilities
transfer likewise.
\end{corollary}

The distinction is worth emphasizing: domination \emph{in expectation} alone would not give the
high-probability transfer. It is the almost-sure coupling that does.

\begin{remark}[What the proof does not use]
Nothing in Steps 1--3 constrains \emph{how} $T$ was chosen, so adaptivity is free; nothing bounds
$|T|$, so unbounded budgets are free; the argument is per-realization, so randomized adversaries are
free; and the learner never needs to identify $S$ inside $T$. Crucially, no exchangeability,
independence, or concentration property of $T$ is invoked. The clean-sample guarantee is applied to
$S$, which \emph{is} i.i.d. This is what makes the result compatible with the fact that the sharpest
clean bounds for monotone rules~\cite{Han16b} genuinely require i.i.d.\ data in their proofs: we
never apply them to the corrupted sample.
\end{remark}

\begin{remark}[Relation to prior monotonicity conditions]\label{rem:equiv}
Hanneke~\cite{Han16b} studies \emph{consistent monotone rules}, where monotonicity is along a
sample prefix: $\psi_{t+1}(z_1,\dots,z_{t+1}) \subseteq \psi_t(z_1,\dots,z_t)$. For
permutation-invariant learners this is \emph{equivalent} to
Definition~\ref{def:is}(ii): given $S \subseteq T$, delete occurrences one at a time, ordering each
deleted occurrence last, and chain the resulting inclusions; permutation-invariance is exactly
what licenses that reordering. Consequently the results of~\cite{Han16b} apply to insertion-stable
learners off the shelf, with no gap to bridge.

The contribution of Theorem~\ref{thm:c1} is therefore not the condition but the \emph{reduction}:
recognizing that a monotone adversary's output is a correctly labeled superset, so that a property
studied for i.i.d.\ prefixes controls a non-exchangeable adversarial sample pointwise. We note the
equivalence does require permutation-invariance, which the shuffle in Problem~\ref{prob:model}
supplies; an order-sensitive prefix-monotone rule need not be insertion-stable.
\end{remark}

\section{Closure is insertion-stable}\label{sec:c2}

\begin{proposition}\label{prop:closure}
For \emph{any} class $\Hy$, the Closure algorithm~\eqref{eq:closure} is insertion-stable.
\end{proposition}

\begin{proof}
Write $P(L) = \bigcap_{h \in V_\Hy(L)}\{x : h(x)=1\}$.

(i) $V_\Hy(L)$ depends on $L$ only through its set of constraints, and is unaffected by their order
or multiplicity, since a hypothesis consistent with one occurrence of $(x,y)$ is consistent with
all. Hence $\hat h_L$ is a function of the labeled multiset.

(ii) Let $S \subseteq T$ be correctly labeled by $\hstar$. Every constraint of $S$ is one of $T$, so
$V_\Hy(T) \subseteq V_\Hy(S)$; and $\hstar$ lies in both, so both are nonempty and Closure is
defined on each. Intersecting over a smaller family yields a larger set, so
\begin{equation}\label{eq:grow}
P(S) \;\subseteq\; P(T).
\end{equation}
Since $\hstar \in V_\Hy(L)$, the intersection defining $P(L)$ includes $\{\hstar = 1\}$, so
$P(L) \subseteq \{x : \hstar(x)=1\}$: Closure never produces a false positive. Its error region is
therefore exactly the missed positives, $\ErrReg_{\hstar}(\hat h_L) = \{\hstar = 1\} \setminus
P(L)$, and by~\eqref{eq:grow},
\[
\ErrReg_{\hstar}(\hat h_T) = \{\hstar{=}1\} \setminus P(T) \;\subseteq\; \{\hstar{=}1\} \setminus
P(S) = \ErrReg_{\hstar}(\hat h_S). \qedhere
\]
\end{proof}

The three edge cases the model permits are all harmless, and for instructive reasons. Repeated
points do not change $V_\Hy(L)$ at all. An out-of-support insertion $(x,\hstar(x))$ is still a
\emph{correct} constraint, so it plays the same role as any other; the argument never mentions $\D$.
And a \emph{negatively} labeled insertion can only shrink $V_\Hy(L)$, which by~\eqref{eq:grow} can
only enlarge $P(L)$: negatives are weakly helpful, never harmful.

\begin{remark}[Stability and rate are separate concerns]\label{rem:factor}
Proposition~\ref{prop:closure} never used intersection-closedness. That hypothesis enters only to
make Closure a \emph{good} learner, guaranteeing its output lies in the intersection-closure
$\bar\Hy$ with $\VC(\bar\Hy) = \VC(\Hy)$~\cite{AO07}, and giving compression size $d$. Stability is
a property of the \emph{learner}; the rate is a property of the \emph{class}.
\end{remark}

\begin{corollary}\label{cor:c2}
Let $\Hy$ be intersection-closed with $\VC(\Hy) = d$ and let $A$ be Closure. Against every
randomized adaptive monotone adversary, at every finite budget $m$ (hence for $R^{\infty}_\D$), with
$n$ clean examples,
\[
\E\big[\Err_\D(\hat h, \hstar)\big] \;\leq\; \frac{21d+34}{n},
\qquad\text{and}\qquad
\Err_\D(\hat h,\hstar) \;\leq\; \frac{1}{n}\Big(21 d + 16\ln\tfrac{3}{\delta}\Big)
\]
with probability at least $1-\delta$.
\end{corollary}

\begin{proof}
Compose Proposition~\ref{prop:closure} with Theorem~\ref{thm:c1} and
Corollary~\ref{cor:transfer}, then apply Hanneke's bound for Closure on intersection-closed
classes~\cite[Theorem 5]{Han16b} to the clean sample $S$, which consists of $n$ i.i.d.\ examples
labeled by $\hstar$.
\end{proof}

Note the constants are inherited rather than re-derived, and that both bounds are free of
logarithmic factors, a point that becomes essential in Theorem~\ref{thm:noscheme}.

\begin{corollary}[The bounded-budget converse is invisible to Closure]\label{cor:c2prime}
Mehrotra's converse~\cite[Remark 3.2]{Meh26} showing that error $O(\min\{1,d/n\})$ remains achievable at $m = O(1)$ proceeds
by augmenting a clean lower-bound class with a common point labeled zero by every concept, and
appending $m$ copies of that point. Both features of the attack are inert against Closure on an
intersection-closed class: the inserted point is labeled \emph{negatively}, and negatives never
shrink $P$; and the insertions are \emph{repeated copies} of one point, which do not change
$V_\Hy(L)$. On intersection-closed classes that padding attack is not merely survivable, it is
invisible.
\end{corollary}

Nothing above is special to Closure. Any learner meeting the hypotheses of Hanneke's compression
bound and Definition~\ref{def:is} inherits the same conclusion, and we record that form because
Section~\ref{sec:c3} needs it.

\begin{corollary}[Scheme-level rate bound]\label{cor:scheme}
Let $A$ be consistent, permutation-invariant, superset-monotone, and expressible as a sample
compression scheme of size $k$ with permutation-invariant reconstruction. Then against every
randomized adaptive monotone adversary, at every finite budget (hence for $R^{\infty}_\D$),
\[
\E\big[\Err_\D(\hat h,\hstar)\big] \leq \frac{21k+34}{n},
\qquad
\Err_\D(\hat h,\hstar) \leq \frac{1}{n}\Big(21k + 16\ln\tfrac{3}{\delta}\Big)
\ \text{ w.p. } \geq 1-\delta .
\]
\end{corollary}
\begin{proof}
Superset-monotonicity and permutation-invariance make $A$ insertion-stable, so Theorem~\ref{thm:c1}
and Corollary~\ref{cor:transfer} reduce the adversarial risk to $A$'s risk on the clean i.i.d.\
sample. There the four hypotheses are exactly those of~\cite[Theorem 3]{Han16b}, whose bounds are
free of logarithmic factors.
\end{proof}

\section{No dimension-based theory can work}\label{sec:c3}

\subsection{The survivor and the victim}

Let $\Hle{d} = \{h : |\{x : h(x)=1\}| \leq d\}$, the concepts labeling at most $d$ points positive
a canonical intersection-closed family~\cite{Han16b}. Let $\Hprime$ be Mehrotra's
Littlestone-two class: for each prime $q$ take a disjoint copy of the projective plane of order $q$,
retain only its \emph{checker} coordinates $c_{rM}$ indexed by incident point--line pairs, set
$u_p(c_{rM}) = \mathbf{1}\{p \in M,\, p \neq r\}$ and $v_L(c_{rM}) = \mathbf{1}\{L = M\}$, extend
each concept by zero outside its component, add the global zero concept, and take the union over
primes.

\begin{lemma}\label{lem:atmostd}
For any domain $\X$ with $|\X| \geq d$, finite or infinite, we have
$\VC(\Hle{d}) = \Ldim(\Hle{d}) = d$, and $\Hle{d}$ is intersection-closed.
\end{lemma}

\begin{proof}
$\VC \geq d$: any $d$ points are shattered, since every subset of a $d$-set has size at most $d$.
$\VC \leq d$: the all-positive pattern on $d+1$ points would need $d+1$ positives. For
$\Ldim \leq d$, recall that $\Ldim$ \emph{is} the optimal mistake bound in the realizable online
model, so exhibiting any one learner with mistake bound $d$ suffices. Take the learner predicting $1$
only at points already known positive. Each mistake occurs where $\hstar$ is positive and reveals a
new element of $\hstar$'s positive set, which has at most $d$ elements, so it errs at most $d$ times.
With $\VC \le \Ldim$ this gives $\Ldim = \VC = d$, independent of $|\X|$. Intersection-closedness is immediate from
$|A \cap B| \le |A| \le d$.
\end{proof}

The mechanism is the intuition for the whole paper: \emph{every positive example spends one unit of
a budget of $d$.} A positive answer is irreversibly informative, so an adversary restricted to
adding correctly labeled points cannot manufacture ambiguity.

For $\Hprime$, Mehrotra proves $\VC(\Hprime) = \Ldim(\Hprime) = 2$ and the lower bound below. We
verified both computationally, constructing $PG(2,q)$ from normalized nonzero triples over $GF(q)$
for $q \in \{2,3\}$ (and the two-prime union), with $\Ldim$ certified by depth-bounded
shattered-tree search: depth $2$ is achievable, depth $3$ is not.

\subsection{The separation}

\begin{theorem}\label{thm:survive}
Against every randomized adaptive monotone adversary at every finite budget $m$ (hence for
$R^{\infty}_\D$), Closure on $\Hle{2}$ satisfies $\E[\Err_\D(\hat h,\hstar)] \leq 76/n$. The rate is $\Theta(1/n)$.
\end{theorem}

\begin{proof}
$\Hle{2}$ is intersection-closed with $\VC = 2$ by Lemma~\ref{lem:atmostd}, so
Corollary~\ref{cor:c2} with $d=2$ gives $(21 \cdot 2 + 34)/n = 76/n$. The adversary may choose
$m=0$, so the rate is at least the clean minimax rate, which is $\Omega(1/n)$ for a class shattering
a pair~\cite{EHKV89}.
\end{proof}

Concretely, Closure on $\Hle{2}$ labels $1$ exactly the positive points it has seen. Its error
region is the set of not-yet-seen positives, of which there are at most two. Every insertion either
reveals a new positive, shrinking the error region, or is a negative or a repeat, changing nothing.
There is nothing the adversary can add that is both correctly labeled and unhelpful.

\begin{theorem}[{\cite[Theorem 3.3, case $d_L = 2$; Lemma 3.16]{Meh26}}]\label{thm:penalty}
For every $n \geq 1$ there is a finite budget $m_n$ such that every randomized, possibly improper
learner for $\Hprime$ has some target, distribution, and deterministic adaptive monotone adversary
with $\E[\Err_\D(\hat h,\hstar)] \geq \log(en)/(30n)$. With the matching $O(\log(en)/n)$ upper bound
by ERM~\cite{LPS26}, the rate for $\Hprime$ is $\Theta(\log(en)/n)$.
\end{theorem}

We invoke this rather than reproving it. Its mechanism is that the adversary can pad two candidate
targets $u_p, v_L$ to the \emph{identical} multiset, leaving the posterior balanced on a point of
test mass $1/(q+1)$ with $q+1 = \Theta(n/\log(en))$.

Theorem~\ref{thm:penalty} is consistent with Corollary~\ref{cor:c2} only because $\Hprime$ is
\emph{not} intersection-closed; had it been, Corollary~\ref{cor:c2} would have given it
$\Theta(1/n)$ and refuted the framework. We verified the failure explicitly: for distinct points
$p \neq p'$ on the line $L = pp'$, the intersection $u_p \cap u_{p'}$ has size $q-1$, and
$|u_p \cap v_L| = q$ for a flag $p \in L$, while the nonempty positive-set sizes realizable in the
class are exactly $q+1$ and $q(q+1)$. Since $1 \le q-1 < q < q+1$ for $q \geq 2$, neither
intersection belongs to the class.

\begin{corollary}[Dimension independence]\label{cor:sep}
$\Hle{2}$ and $\Hprime$ both have $\VC = \Ldim = 2$, with monotone-insertion rates $\Theta(1/n)$ and
$\Theta(\log(en)/n)$ respectively. Hence no function of VC and Littlestone dimension determines the
monotone-insertion rate.
\end{corollary}

Two remarks. First, $d = 2$ is the right place for this: at $d = 1$ the rate is $\Theta(1/n)$
regardless~\cite[Theorem 3.1]{Meh26}, so there is nothing to separate. Second, the separation runs in
\emph{both} directions. Intervals on $\R$ are intersection-closed with $\VC = 2$, so by
Corollary~\ref{cor:c2} they retain the clean rate under arbitrarily many insertions, yet their
Littlestone dimension is \emph{unbounded}, growing as $\Theta(\log N)$ on an ordered domain of $N$
points (thresholds alone already give $\lfloor \log_2 N\rfloor$). So bounded Littlestone dimension
is neither necessary nor sufficient for surviving monotone insertions.

\begin{corollary}[A barrier, not a non-existence]\label{cor:barrier}
Every insertion-stable learner for $\Hprime$ has \emph{clean-sample} rate $\Omega(\log(en)/n)$.
Equivalently, on $\Hprime$ no learner is both insertion-stable and clean-rate-optimal.
\end{corollary}
\begin{proof}
Writing $\mathrm{CR}(A,n)$ for $A$'s clean-sample rate, Corollary~\ref{cor:transfer} gives
$\sup_m \E[\Err_\D(A(T))] \leq \mathrm{CR}(A,n)$, while Theorem~\ref{thm:penalty} gives
$\sup_m \E[\Err_\D(A(T))] \geq \log(en)/(30n)$. Chain the two.
\end{proof}

The stronger-sounding claim that \emph{no} insertion-stable learner exists for $\Hprime$ would be
false: the constant learner $A(L) \equiv h_0$ is insertion-stable on every class, its error region
never changing. Insertion-stability constrains how the error region \emph{moves}, never how
\emph{small} it is, so it never implies a good rate on its own; every rate statement here pairs it
with a clean-sample guarantee for a specific learner.

\subsection{The negative half, sharpened to a family of algorithms}\label{sec:noscheme}

Corollary~\ref{cor:barrier} rules out one property of one learner at a time. The following says more:
for $\Hprime$, an entire natural family of algorithms fails, at every size. This is the one result
here whose proof we do not think an expert reconstructs on sight, and it is stated about
\emph{schemes}, so no dimension is needed to say it.

\begin{theorem}[No monotone compression scheme certifies $\Hprime$]\label{thm:noscheme}
There is no $k < \infty$ and no learner for $\Hprime$ that is simultaneously consistent,
permutation-invariant, superset-monotone, and expressible as a sample compression scheme of size $k$
with permutation-invariant reconstruction.
\end{theorem}
\begin{proof}
Suppose such a learner existed for some finite $k$. Corollary~\ref{cor:scheme} would give it
$\E[\Err_\D] \leq (21k+34)/n$ against every adaptive monotone adversary at every finite budget. But
$\Hprime$ is a \emph{single fixed} class for which Theorem~\ref{thm:penalty} supplies, at every $n$,
an adversary forcing $\E[\Err_\D] \geq \log(en)/(30n)$ on every learner. Choosing $n$ large enough
that $\log(en)/30 > 21k+34$ contradicts the first bound. As $k$ was arbitrary, no finite $k$ works.
\end{proof}

Three features of this argument are load-bearing, and each fails if weakened.

\begin{itemize}
\item \textbf{The compression bound must be free of logarithmic factors.} The contradiction pits a
constant, $21k+34$, against a growing $\log(en)/30$. The generic sample-compression estimate
$O\big(k\log(n/k)/n\big)$ would instead \emph{exceed} Theorem~\ref{thm:penalty} for every $k \geq 1$
at every $n$, and the argument would collapse entirely. Hanneke's Theorem 3 is log-free, which is
precisely why the hypothesis list in Corollary~\ref{cor:scheme} is the one it is.
\item \textbf{The bound must not depend on the budget.} Theorem~\ref{thm:penalty}'s witnessing
budget grows rapidly with $n$, so a budget-dependent guarantee such as Mehrotra's own
$O(d(m{+}1)/n)$ is vacuous in the relevant regime. Superset-monotonicity is what removes $m$ from
the bound (Remark~\ref{rem:sup}).
\item \textbf{The class must be fixed, not $n$-dependent.} Theorem 3.5 of~\cite{Meh26} yields a
different finite class for each $n$, from which no statement about a single class would follow.
$\Hprime$'s union over primes is one class that is hard at every $n$, which is what licenses
``no finite $k$''.
\end{itemize}

\begin{table}[t]
\centering
\small
\setlength{\tabcolsep}{4.5pt}
\begin{tabular}{lccclc}
\toprule
Class & $\VC$ & $\Ldim$ & $\isdim$ & rate vs.\ adversary & \shortstack{stable\\optimal?} \\
\midrule
$\Hle{2}$ & $2$ & $2$ & $\leq 2$ & $\Theta(1/n)$, $\leq 76/n$ & yes\,$^{b}$ \\
$\Hprime$~\cite{Meh26} & $2$ & $2$ & $\infty^{a}$ & $\Theta(\log(en)/n)$ & no\,$^{a}$ \\
$\Hle{d}$ & $d$ & $d$ & $\leq d$ & $\Theta(d/n)$ & yes\,$^{b}$ \\
axis-aligned rectangles, $\R^p$ & $2p$ & $\infty$ & $\leq 2p$ & $\Theta(p/n)$ & yes\,$^{b}$ \\
monotone conjunctions, $\{0,1\}^p$ & $p$ & $p$ & $\leq p$ & $\Theta(p/n)$ & yes\,$^{b}$ \\
linear subspaces, $\dim k$ & $k$ & $k$ & $\leq k$ & $\Theta(k/n)$ & yes\,$^{b}$ \\
intervals on $\R$ & $2$ & $\infty$ & $\leq 2$ & $\Theta(1/n)$ & yes\,$^{b}$ \\
general intersection-closed & $d$ & --- & $\leq d$ & $\Theta(d/n)$ & yes\,$^{b}$ \\
halfspaces in $\R^p$~\cite{BHQS21} & $p{+}1$ & $\infty$ & $?^{c}$ & --- & $?^{c}$ \\
general $\VC = d$, worst case~\cite{Meh26} & $d$ & --- & $\infty^{a}$ &
$\Theta\big(\tfrac{d}{n}\log\tfrac{n}{d}\big)$ & no\,$^{a}$ \\
\bottomrule
\end{tabular}
\caption{The rate landscape. The first two rows are the payload: identical VC and Littlestone
dimension, rates differing by $\Theta(\log n)$. Rows four and seven run the separation the other way:
unbounded Littlestone dimension, clean rate retained. The last column asks whether the class admits an
insertion-stable learner \emph{attaining its clean rate}, not merely an insertion-stable learner,
since the constant learner $A(L)\equiv h_0$ is insertion-stable on every class and useful on none.
Provenance of each $\isdim$ and last-column entry: $^{a}$ from a rate lower bound via the
contrapositive of Corollary~\ref{cor:scheme} (in the last row, for the witness class attaining the
supremum); $^{b}$ from an explicit scheme, namely Closure (Proposition~\ref{prop:isdimupper});
$^{c}$ \textbf{open}. Failing to be intersection-closed defeats \emph{Closure} here, but does not
preclude some other superset-monotone permutation-invariant scheme, and we prove no rate lower
bound for halfspaces. That entry illustrates the limitation noted in
Section~\ref{sec:open}.}
\label{tab:landscape}
\end{table}

\begin{figure}[t]
\centering
\includegraphics[width=\textwidth]{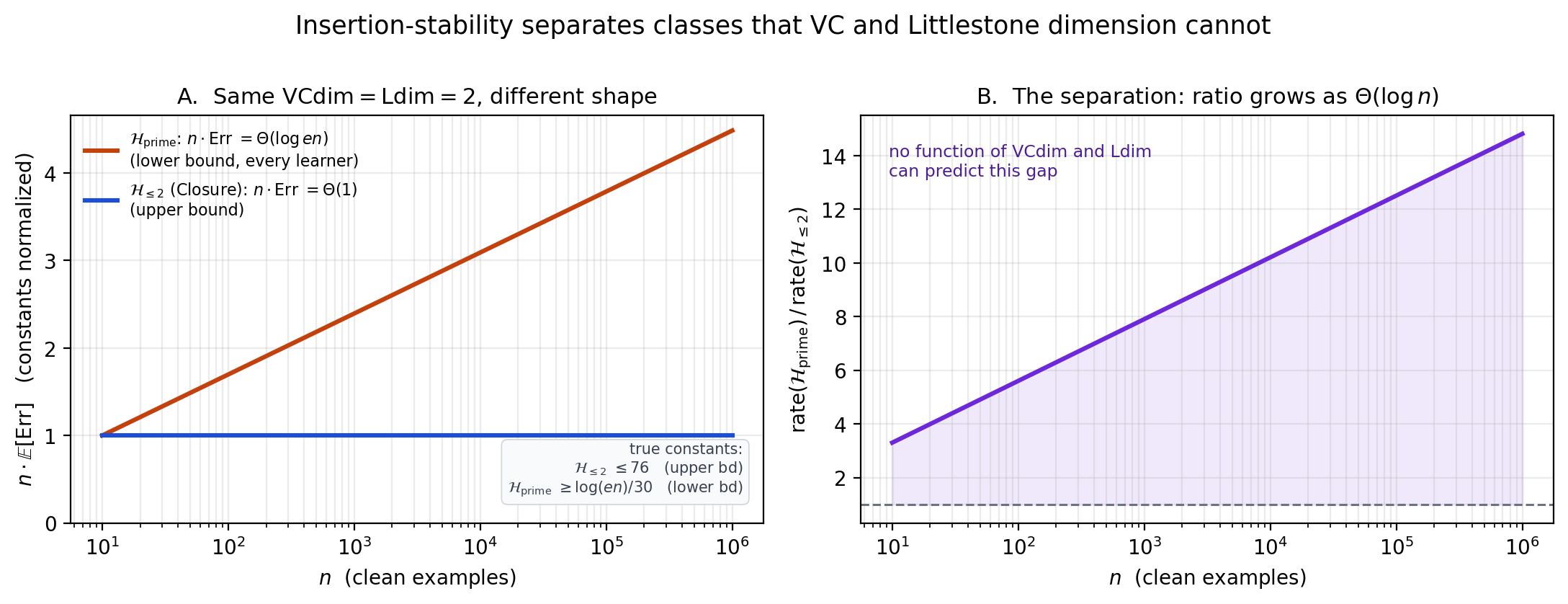}
\caption{Left: $\Hle{2}$ and $\Hprime$ have identical VC and Littlestone dimension but different
asymptotic shape, with constants normalized (the true bounds are $\leq 76$, an upper bound, and
$\geq \log(en)/30$, a lower bound; these are not pointwise comparable and their literal crossover
lies beyond any plottable range). Right: the ratio of rates, which is the content of
Corollary~\ref{cor:sep}, diverging as $\Theta(\log n)$.}
\label{fig:sep}
\end{figure}

\section{A candidate dimension, and why it is hard to compute}\label{sec:c4}

One may package Corollary~\ref{cor:scheme}'s hypotheses as a measure: let $\isdim(\Hy)$ be the least
$k$ admitting such a size-$k$ scheme, and $\infty$ if there is none. Then Corollary~\ref{cor:scheme}
reads $\E[\Err] \leq (21\isdim(\Hy)+34)/n$ at every finite budget, and Theorem~\ref{thm:noscheme} reads
$\isdim(\Hprime) = \infty$. On intersection-closed classes it is pinned down to a constant:

\begin{proposition}\label{prop:isdimupper}
For intersection-closed $\Hy$, writing $v = \VC(\Hy)$, asymptotically
\[
c_0\, v \;\leq\; \isdim(\Hy) \;\leq\; v ,
\]
so $\isdim(\Hy) = \Theta(v)$. Here $\Htwo$ denotes binary entropy and $c_0 \approx 0.2271$ solves
$\Htwo(c)+c=1$.
\end{proposition}
\begin{proof}
\emph{Upper.} Closure satisfies the first three conditions by Proposition~\ref{prop:closure}, and for
every $L$ at most $\VC(\Hy)$ of the observed \emph{positive} points share its closure~\cite{AO07}, so
Closure on those points is a scheme of that size. \emph{Lower.} Let $C$ be shattered, $|C| = v$. The
$2^v$ samples listing $C$ under each labeling must receive $2^v$ distinct messages, since equal
messages give equal classifiers and one classifier cannot agree with two labelings of $C$. Counting
labeled subsets of size $\leq k$ gives $2^v \leq \sum_{j \leq k}\binom{v}{j}2^j$, whose dominant term
at $k = cv$ is $2^{v(\Htwo(c)+c)}$, forcing $\Htwo(c)+c \geq 1$.
\end{proof}

\begin{remark}[What this measure is and is not]\label{rem:nolower}
We state $\isdim$ for bookkeeping and deliberately build nothing on it, for three reasons a reader
should weigh before adopting it. \emph{First, its upper bound is immediate by construction rather than
by insight}: the defining conditions are verbatim the hypotheses of~\cite[Theorem 3]{Han16b}, so
``$\isdim(\Hy) = k$ implies rate $O(k/n)$'' restates that theorem. \emph{Second, it has no known
intrinsic evaluation method.} Every finite value we can supply comes from exhibiting Closure, and
every infinite value from a pre-existing rate lower bound (Theorem~\ref{thm:noscheme}); a class with
neither, such as halfspaces, cannot be placed at all, which is why
Table~\ref{tab:landscape} carries a ``$?$''. \emph{Third, the lower constant in
Proposition~\ref{prop:isdimupper} cannot be raised to $1$.} Take $\Hy$ to be all $2^3$ functions
on a $3$-point domain, an intersection-closed class with $\VC = 3$. Closure on this class is
consistent, permutation-invariant, and superset-monotone (Proposition~\ref{prop:closure} applies to
every class), and it admits a compression rule of size $2$: retain the underlying labeled set when
some point is unobserved; when all three points are observed and some label is $0$, drop one
zero-labeled point (reconstruction predicts $0$ at unretained points, so no collision arises); the
all-positive sample of size three retains nothing, and at larger sizes retains a repeated positive
occurrence, and these keys are selected by no other sample. Since
$2^3 = 8 > 7 = \sum_{j\leq 1}\binom{3}{j}2^j$ rules out size $1$, this class has $\isdim = 2 < 3 =
\VC$ exactly, with all four conditions of the definition verified exhaustively (consistency and
monotonicity over all correctly-labeled multiset pairs, compression bookkeeping over all samples up
to size six; the learner is a function of the underlying labeled set, so larger multiplicities
change nothing). More generally the compressor's \emph{choice} of retained points is itself a
channel worth up to $\log_2\binom{v}{j}$ bits, which defeats the tempting argument that some
unretained point forces a collision: the scheme may retain different subsets for different samples
precisely to avoid one.
\end{remark}

The separations of Section~\ref{sec:c3} can be read off this quantity, but nothing is gained by doing
so: they are Corollary~\ref{cor:sep} and Theorem~\ref{thm:noscheme}, statements about rates and
schemes that need no dimension.

\section{Discussion and open problems}\label{sec:open}

Our results are one-directional: $\isdim(\Hy) = k$ certifies the rate $O(k/n)$, but we do not show
that large $\isdim$ \emph{forces} a large rate. The $\isdim = \infty$ entries in
Table~\ref{tab:landscape} are derived \emph{from} known rate lower bounds via
Corollary~\ref{cor:scheme} rather than computed intrinsically. $\isdim$ currently
certifies good rates and merely records bad ones.

\begin{openproblem}
Is there a combinatorial quantity $q(\Hy)$, computable from $\Hy$ without reference to any rate,
such that $q(\Hy)$ large implies every learner suffers $\omega(1/n)$ under monotone insertions?
Ideally $q = \isdim$, giving a characterization.
\end{openproblem}

The obstacle is structural: $\isdim$ is defined by an \emph{existential over learners}, so proving
$\isdim(\Hy) = \infty$ directly means excluding every superset-monotone compression scheme of every
size. Every dimension in this literature that supports a lower bound, whether VC, Littlestone, the
extended threshold dimension of~\cite{DFHS26}, or inference dimension, is instead defined by an
explicit shattering or witness structure. This suggests seeking a shattering-style quantity first
and relating it to $\isdim$ afterwards. Mehrotra's construction indicates the shape: many pairs of
concepts agreeing except at a point of non-negligible mass, together with a correctly labeled
completion consistent with both. Any candidate must vanish on intersection-closed classes, or it
contradicts Corollary~\ref{cor:c2}. This is a cheap and strong test.

\begin{openproblem}\label{op:lower}
Which constants are optimal in $\isdim = \Theta(\VC)$ for intersection-closed classes?
Proposition~\ref{prop:isdimupper} gives $0.2271\,\VC \leq \isdim \leq \VC$, but neither bound is
known to be attained. Closure realizes compression size $\VC$, which certifies $\isdim \leq \VC$ and
never $\isdim = \VC$, since $\isdim$ minimizes over all schemes; and Remark~\ref{rem:nolower}
exhibits an intersection-closed class with $\isdim/\VC = 2/3$. Let $c_-$ and $c_+$ be the least and
the greatest value of the ratio $\isdim/\VC$ across such classes. The witness then gives
\[
0.2271 \;\leq\; c_- \;\leq\; 2/3 \;\leq\; c_+ \;\leq\; 1 ,
\]
and all four inequalities may be strict.
\end{openproblem}

\begin{openproblem}
What is $\isdim$ on classes that are neither intersection-closed nor of the $\Hprime$ type?
Proposition~\ref{prop:isdimupper} brackets it on intersection-closed classes and
Theorem~\ref{thm:noscheme} gives $\infty$ for $\Hprime$, but the middle of the landscape is
uncharted in both directions, and Remark~\ref{rem:nolower} explains why we have no method that
would settle a new case.
\end{openproblem}

Finally, Corollary~\ref{cor:transfer} yields high-probability guarantees for free wherever a clean
high-probability bound exists, which bears on Mehrotra's last stated open problem, the optimal
$(\varepsilon,\delta)$ sample complexity under monotone insertions; the matching lower bound is open.

\paragraph{Related work.} Beyond~\cite{LPS26,Meh26}, the closest methodological relative
is~\cite{DFHS26}, which studies online learning under a \emph{replay} adversary, introduces the
extended threshold dimension with matching bounds, uses a closure-based learner, and characterizes
proper learnability by being ``(almost) intersection-closed''; they also exhibit classes of constant
Littlestone dimension with unbounded extended threshold dimension. Their model is
online with corrupted \emph{feedback} that may be wrong, whereas ours is offline with corrupted
\emph{training data} that is always correct, and our risk is measured on the untouched clean
distribution; and they obtain matching bounds where we have only an upper bound. Their lower-bound
technique is the most promising source for the converse above. Also close
is~\cite{EG26}, which studies clean-label adversarial injections with abstention and introduces a
certificate dimension via ``robust witnesses'', compact labeled subsets certifying predictions
while resisting contamination, structurally close to our compression schemes. Notably their positive
results \emph{require} abstention, without which halfspaces are not learnable~\cite{BHQS21}; the
monotone constraint lets Closure predict everywhere and still be immune. The name ``monotone
adversary'' comes from semirandom models in combinatorial optimization~\cite{Fei21}, where
``helpful'' modifications likewise raise recovery thresholds.

\end{document}